\documentclass[nohyperref]{article}

\usepackage{natbib}  

\usepackage{xcolor}         

\usepackage[utf8]{inputenc} 
\usepackage[T1]{fontenc}    
\usepackage{hyperref}       
\usepackage{url}            
\usepackage{booktabs}       
\usepackage{amsfonts}       
\usepackage{nicefrac}       
\usepackage{microtype}      
\usepackage{xcolor}
\usepackage{graphicx}
\usepackage{subfigure}
\usepackage{times}		
\usepackage[T1]{fontenc}
\usepackage[utf8]{inputenc} 
\DeclareUnicodeCharacter{00A0}{~}
\usepackage{amssymb,amsmath,amscd,amsfonts,amsthm,bbm,mathrsfs,yhmath}

\usepackage{hyperref}
\usepackage{caption}
\usepackage{graphics}
\usepackage{mathtools}
\usepackage{cleveref}
\usepackage{comment}

\usepackage{paralist,enumitem}                                       
\usepackage{pdfpages}

\usepackage[textwidth=2cm, textsize=footnotesize]{todonotes}

\newcommand{\cF}{{\mathcal F}}

\newcommand{\cO}{{\mathcal O}}

\newcommand{\cN}{{\mathcal N}}

\newcommand{\cP}{{\mathcal P}}
\newcommand{\X}{{\mathcal X}} 
\newcommand{\Y}{{\mathcal Y}} 

\newcommand{\cH}{{\mathcal H}}

\newcommand{\cS}{{\mathcal S}} 
 
\newcommand{\E}{{{\mathbb E}}}
 
\newcommand{\R}{{\mathbb R}} 
 
\newcommand{\bE}{{\mathbb E}}

\newcommand{\KL}{\mathop{\mathrm{KL}}\nolimits}
\newcommand{\KSD}{\mathop{\mathrm{KSD}}\nolimits}

\newcommand{\HS}{\mathop{\mathrm{HS}}\nolimits}
\newcommand{\op}{\mathop{\mathrm{op}}\nolimits}

\newcommand{\diver}{\mathop{\mathrm{div}}\nolimits}
\newcommand{\st}{\mathop{\mathrm{Stein}}\nolimits}
\newcommand{\ps}[1]{\langle #1 \rangle}

\usepackage[accepted]{icml2022}

\theoremstyle{plain}
\newtheorem{theorem}{Theorem}[section]
\newtheorem{proposition}[theorem]{Proposition}
\newtheorem{lemma}[theorem]{Lemma}
\newtheorem{corollary}[theorem]{Corollary}
\theoremstyle{definition}
\newtheorem{definition}[theorem]{Definition}
\newtheorem{assumption}[theorem]{Assumption}
\theoremstyle{remark}

\icmltitlerunning{Convergence of SVGD in the Population Limit under T1}

\begin{document}

\twocolumn[
\icmltitle{A Convergence Theory for SVGD in the Population Limit \\ under Talagrand's Inequality T1}




\begin{icmlauthorlist}
\icmlauthor{Adil Salim}{MSR}
\icmlauthor{Lukang Sun}{KAUST}
\icmlauthor{Peter Richtárik}{KAUST}
\end{icmlauthorlist}

\icmlaffiliation{MSR}{Microsoft Research, Redmond, USA}
\icmlaffiliation{KAUST}{King Abdullah University of Science and Technology, Thuwal, Saudi Arabia}

\icmlcorrespondingauthor{Adil Salim}{adilsalim@microsoft.com}

\icmlkeywords{Stein Variational Gradient Descent, Complexity, Convergence, Optimization, Sampling}

\vskip 0.3in
]



\printAffiliationsAndNotice{}  

\begin{abstract}
Stein Variational Gradient Descent (SVGD) is an algorithm for sampling from a target density which is known up to a multiplicative constant. Although SVGD is a popular algorithm in practice, its theoretical study is limited to a few recent works. We study the convergence of SVGD in the population limit, (i.e., with an infinite number of particles) to sample from a non-logconcave target distribution satisfying Talagrand's inequality T1. We first establish the convergence of the algorithm. Then, we establish a dimension-dependent complexity bound in terms of the Kernelized Stein Discrepancy (KSD). Unlike existing works, we do not assume that the KSD is bounded along the trajectory of the algorithm. Our approach relies on interpreting SVGD as a gradient descent over a space of probability measures. 
\end{abstract}


\section{Introduction}
Sampling from a given target distribution $\pi$ is a fundamental task of many Machine Learning procedures. In Bayesian Machine Learning, the target distribution $\pi$ is typically known up to a multiplicative factor and often takes the form
\begin{equation}
    \label{eq:target}
    \pi(x) \propto \exp(-F(x)),
\end{equation}
where $F:\X \to \R$ is a $L$-smooth nonconvex function defined on $\X \coloneqq \R^d$, and satisfying $$\int \exp(-F(x))dx < \infty.$$

As sampling algorithms are intended to be applied to large scale problems, it has become increasingly important to understand their theoretical properties, such as their complexity, as a function of the \textit{dimension of the problem} $d$, and the \textit{desired accuracy} $\varepsilon$. In this regard, most of the Machine Learning literature on sampling has concentrated on understanding the complexity (in terms of $d$ and $\varepsilon$) of (variants of) the Langevin algorithm, see \citet{durmus2018analysis,bernton2018langevin,wibisono2018sampling,cheng2017underdamped,salim2020primal,hsieh2018mirrored,dalalyan2017theoretical,durmus2017nonasymptotic,rolland2020double,vempala2019rapid,zou2019sampling,csimvsekli2017fractional,shen2019randomized,bubeck2018sampling,durmus2018efficient,ma2019there,foster2021shifted,li2021sqrt}. 

\subsection{Stein Variational Gradient Descent (SVGD)}

Stein Variational Gradient Descent (SVGD)~\citep{liu2016stein,liu2017stein} is an alternative to the Langevin algorithm that has been applied in several contexts in Machine Learning, including Reinforcement Learning~\citep{liu2017policy}, sequential decision making~\citep{zhang2018learning,zhang2019scalable}, Generative Adversarial Networks~\citep{tao2019variational}, Variational Auto Encoders~\citep{pu2017vae}, and Federated Learning~\citep{kassab2020federated}. 

The literature on theoretical properties of SVGD is scarce compared to that of Langevin algorithm, and limited to a few recent works~\citep{korba2020non,lu2019scaling,duncan2019geometry,liu2017stein,chewi2020svgd,gorham2020stochastic,nusken2021stein,shi2021sampling}. In this paper, our goal is to provide a clean convergence theory for SVGD in the population limit, \textit{i.e.}, with an infinite number of particles.

\subsection{Related works}

The Machine Learning literature on the complexity of sampling from a non-logconcave target distribution has mainly focused on the Langevin algorithm. For instance,\footnote{The example of~\citet{vempala2019rapid} is taken only for illustration purpose. Many other results were obtained for Langevin algorithm, even in nonconvex cases, see above.} \citet{vempala2019rapid} showed that Langevin algorithm reaches $\varepsilon$ accuracy in terms of the Kullback-Leibler divergence after $\tilde{\Omega}(\frac{L^2 d}{\lambda^2 \varepsilon})$ iterations, assuming that the target distribution satisfies the logarithmic-Sobolev inequality (LSI) with constant $\lambda$. In this work, we will assume Talagrand's inequality T1 with constant $\lambda$, which is milder than LSI with constant $\lambda$, and we will prove a complexity result in terms of another discrepancy called Kernelized Stein Discrepancy (KSD). Besides, a very recent work studies Langevin algorithm for a non-logconcave target distribution without assuming LSI and provides guarantees in terms of the Fisher information~\cite{balasubramanian2022towards}.


Most existing results on SVGD deal with the \textit{continuous time} approximation of SVGD in the \textit{population limit}, a Partial Differential Equation (PDE) representing SVGD with a vanishing step size and an infinite number of particles~\citep{lu2019scaling,duncan2019geometry,liu2017stein,nusken2021stein,chewi2020svgd}. In particular, \citet{duncan2019geometry} propose a Stein logarithmic Sobolev inequality that implies the linear convergence of this PDE. However, it is not yet understood when Stein logarithmic Sobolev inequality holds. Besides, \citet{chewi2020svgd} showed that the Wasserstein gradient flow of the chi-squared divergence can be seen as an approximation of that PDE, and showed linear convergence of the Wasserstein gradient flow of the chi-squared under Poincaré inequality. Other results, such as those of~\citet{lu2019scaling,liu2017stein,nusken2021stein}, include asymptotic convergence properties of the PDE, but do not include convergence rates. In this paper, we will prove convergence rates for SVGD in discrete time.

\subsubsection{Comparison to~\citet{korba2020non}}


The closest work to ours is~\citet{korba2020non}. To our knowledge, \citet{korba2020non} showed the first complexity result for SVGD in \textit{discrete time}. This result is proven in the \textit{population limit} and in terms of the Kernelized Stein Discrepancy (KSD), similarly to our main complexity result. 

However, their complexity result relies on the assumption that the KSD is \textit{uniformly bounded} along the iterations of SVGD, an assumption that cannot be checked prior to running the algorithm. Moreover, their complexity bound does not express \textit{the dependence in the dimension} $d$ explicitly. This is because the uniform bound on the KSD appears in their complexity bound. On the contrary, one of our contributions is to present a dimension-dependent complexity result under verifiable assumptions.

Besides, \citet{korba2020non} provide a bound on the distance between SVGD in the finite number of particles regime and SVGD in the population limit. This bound cannot be used to study the complexity or convergence rate of SVGD in the finite number of particles regime, see~\citet[Proposition 7]{korba2020non}.

\subsection{Contributions}
\label{sec:contrib}
We consider SVGD in the population limit, similarly to concurrent works such as~\citet{liu2017stein,korba2020non,gorham2020stochastic}.
Our paper intends to provide a clean analysis of SVGD, a problem stated in~\citet[Conclusion]{liu2017stein}. To this end, we do not make any assumptions on the trajectory of the algorithm. Instead, our key assumption is that the target distribution $\pi$ satisfies T1, the mildest of the Talagrand's inequalities, which holds under a mild assumption on the tails of the distribution; see~\citet[Theorem 22.10]{villani2008optimal}. Moreover, T1 is implied, for example, by the logarithmic Sobolev inequality~\citep[Theorem 22.17]{villani2008optimal}, with the same constant $\lambda$.

Although sampling algorithms are meant to be applied on high-dimensional problems, the question of the dependence of the complexity of SVGD in $d$ has not been studied in concurrent works, nor has been studied the generic weak convergence of SVGD under verifiable assumptions, to our knowledge. Assuming that the T1 inequality holds, we provide 
\begin{itemize}
    \item a generic weak convergence result for SVGD (actually our result is a bit stronger: convergence holds in 1-Wasserstein distance), 
    \item a complexity bound for SVGD in terms of the dimension $d$ and the desired accuracy $\varepsilon$, under verifiable assumptions (i.e., assumptions that do not depend on the trajectory of the algorithm): $\tilde{\Omega}\left(\frac{L{d}^{3/2}}{\lambda^{1/2}\varepsilon}\right)$ iterations suffice to obtain a sample $\mu$ such that $\KSD^2(\mu|\pi) < \varepsilon$, where $L$ is the smoothness constant of $F$ and $\lambda$ the constant in T1 inequality.
\end{itemize}
Note that these results hold without assuming $F$ convex. In particular, in the population limit, SVGD applied to \textit{non-logconcave} target distributions satisfying T1 converges to the target distribution.


\subsection{Paper structure}
The remainder of the paper is organized as follows. In Section~\ref{sec:background} we introduce the necessary mathematical and notational background on optimal transport, reproducing kernel Hilbert spaces and SVGD in order to be able to describe and explain our results. Section~\ref{sec:svgd} is devoted to the development of our theory. Finally, in Section~\ref{sec:weak} we formulate three corollaries of our key result, capturing weak convergence and complexity estimates for SVGD. Technical proofs are postponed to the Appendix.

\section{Background and Notation} \label{sec:background}
 
\subsection{Notation}

For any Hilbert space $H$, we denote by $\ps{\cdot,\cdot}_{H}$ the inner product of $H$ and by $\|\cdot\|_{H}$ its norm.

We denote by $C_0(\X)$ the set of continuous functions from $\X$ to $\R$ vanishing at infinity and by $C^1(\X,\Y)$ the set of continuously differentiable functions from $\X$ to a Hilbert space $\Y$. Given $\phi \in C^1(\X,\R)$, its gradient is denoted by $\nabla \phi$, and if $\phi \in C^1(\X,\X)$, the Jacobian of $\phi$ is denoted by $J \phi$. For every $x \in \X$, $J \phi(x)$ can be seen as a $d \times d$ matrix. The trace of the Jacobian, also called divergence, is denoted by $\diver \phi$. 

For any $d \times d$ matrix $A$, $\|A\|_{\HS}$ denotes the Hilbert Schmidt norm of $A$ and $\|A\|_{\op}$ the operator norm of $A$ viewed as a linear operator $A: \X \to \X$ (where $\X$ is endowed with the standard Euclidean inner product). Finally, $\delta_x$ is the Dirac measure at $x \in \X$.

\subsection{Optimal transport}

Consider $p \geq 1$. We denote by $\cP_p(\X)$ the set of Borel probability measures $\mu$ over $\X$ with finite $p^{\text{th}}$ moment: $\int \|x\|^p d\mu(x) < \infty$. We denote by $L^p(\mu)$ the set of measurable functions $f : \X \to \X$ such that $\int \|f\|^p d\mu < \infty$. Note that the identity map $I$ of $\X$ satisfies $I \in L^p(\mu)$ if $\mu \in \cP_p(\X)$. Moreover, denoting the image (or pushforward) measure of $\mu$ by a map $T$ as $T \# \mu$, we have that if $\mu \in \cP_p(\X)$ and $T \in L^p(\mu)$ then $T \# \mu \in \cP_p(\X)$ using the transfer lemma.

For every $\mu,\nu \in \cP_p(\X)$, the $p$-Wasserstein distance between $\mu$ and $\nu$ is defined by
\begin{equation}
    \label{eq:wass}
    W_p^p(\mu,\nu) = \inf_{s \in \cS(\mu,\nu)} \int \|x-y\|^p ds(x,y),
\end{equation}
where $\cS(\mu,\nu)$ is the set of couplings between $\mu$ and $\nu$, \textit{i.e.}, the set of nonnegative measures over $\X^2$ such that $P \# s = \mu$ (resp. $Q \# s = \nu$) where $P: (x,y) \mapsto x$ (resp. $Q: (x,y) \mapsto y$) denotes the projection onto the first
(resp. the second) component. The $p$-Wasserstein distance is a metric over $\cP_p(\X)$. The metric space $(\cP_2(\X),W_2)$ is called the Wasserstein space.

In this paper, we consider a target probability distribution $\pi$ proportional to $\exp(-F)$, where $F$ satisfies the following.
\begin{assumption}
\label{ass:V_Lipschitz} The Hessian $H_{F}$ is well-defined and $\exists L \geq 0$ such that $\|H_{F}\|_{\op} \le L$.
\end{assumption}
Moreover, using $\int \exp(-F(x))dx < \infty$, $F$ admits a stationary point.
\begin{proposition}
\label{prop:stationary} Under Assumptions~\ref{ass:V_Lipschitz} ,
there exists $x_\star \in \X$ for which $\nabla F(x_\star) = 0$, \textit{i.e.}, $F$ admits a stationary point.
\end{proposition}
To specify the dependence in the dimension of our complexity bounds, we will initialize the algorithm from a Gaussian distribution centered at a stationary point. Such a stationary point can be found by gradient descent on $F$ for instance. 

The task of sampling from $\pi$ can be formalized as an optimization problem. Indeed, define the Kullback-Leibler ($\KL$) divergence as 
\begin{equation}
\label{eq:KL}
    \KL(\mu|\pi) \coloneqq \int \log\left(\frac{d\mu}{d\pi}(x)\right)d\mu(x),
\end{equation}
if $\mu$ admits the density $\frac{d\mu}{d\pi}$ with respect to $\pi$, and $\KL(\mu|\pi) \coloneqq +\infty$ else. Then, $\KL(\mu|\pi) \geq 0$ and $\KL(\mu|\pi) = 0$ if and only if $\mu = \pi$. Therefore, assuming $\pi \in \cP_2(\X)$, the optimization problem
\begin{equation}
    \label{eq:optim-pb}
    \min_{\mu \in \cP_2(\X)} \cF(\mu),
\end{equation}
where 
\begin{equation*}
    \cF(\mu) \coloneqq \KL(\mu|\pi),
\end{equation*}
admits a unique solution: the distribution $\pi$.
We will see in Section~\ref{sec:svgd} that SVGD can be seen as a gradient descent algorithm to solve~\eqref{eq:optim-pb}.

Indeed, the Wasserstein space can be endowed with a differential structure. In particular, when it is well defined, the Wasserstein gradient of the functional $\cF$ denoted by $\nabla_W \cF(\mu)$ is an element of $L^2(\mu)$ and satisfies $\nabla_W \cF(\mu) = \nabla \log \left(\frac{d\mu}{d\pi}\right)$.

\subsubsection{Functional inequalities}

The analysis of sampling algorithm in the case where $F$ is nonconvex often goes through functional inequalities.

\begin{definition}[Logarithmic Sobolev Inequality (LSI)]
The distribution $\pi$ satisfies the Logarithmic Sobolev Inequality if there exists $\lambda > 0$ such that  for all $\mu \in \cP_2(\X)$, $$\cF(\mu) \leq \frac{2}{\lambda}\|\nabla_W \cF(\mu)\|_{L^2(\mu)}^2.$$
\end{definition}

LSI is a popular assumption in the analysis of Langevin algorithm in the case when $F$ is not convex see \textit{e.g.}~\citet{vempala2019rapid}.

\begin{definition}[Talagrand's Inequality T$p$]
Let $p \geq 1$. The distribution $\pi$ satisfies the Talagrand's Inequality T$p$ if there exists $\lambda > 0$ such that for all $\mu \in \cP_p(\X)$, we have $W_p(\mu,\pi) \leq \sqrt{\frac{2 \cF(\mu)}{\lambda}}$.
\end{definition}

We now claim that T1 is milder than LSI. Indeed, using $W_1(\mu,\pi) \leq W_2(\mu,\pi)$, T2 implies T1 with the same constant $\lambda$. Moreover, using~\citet[Theorem 22.17]{villani2008optimal}, LSI implies T2 with the same constant $\lambda$. In conclusion, LSI $\Rightarrow$ T2 $\Rightarrow$ T1, with the same constant $\lambda$. 

Besides, if $F$ is $\lambda$-strongly convex, then $\pi$ satisfies LSI with constant $\lambda$. A bounded perturbation of $\pi$ in the latter case would also satisfies LSI with a constant independent of the dimension~\citep[Remark 21.5]{villani2008optimal}.

Finally, to get the exponential convergence of SVGD in continuous time, another inequality called Stein-LSI was proposed in~\citet{duncan2019geometry}. Stein-LSI is an assumption on both the kernel and the target distribution, and it implies LSI. Obtaining reasonable sufficient conditions for Stein-LSI to hold is an open problem, but there are simple cases where it cannot hold~\citep[Lemma 36]{duncan2019geometry}. In particular, Stein-LSI never holds under the assumptions that we will make in this paper to study SVGD in discrete time, see~\citet[Section 11.3]{korba2020non}.

Our key assumption on $\pi$ is that it satisfies the Talagrand's inequality T1~\citep[Definition 22.1]{villani2008optimal}.
\begin{assumption}
\label{ass:T1} 
The target distribution $\pi$ satisfies T1.
\end{assumption}
We will use~\Cref{ass:T1} to recursively control the KSD by the KL divergence along the iterations of the algorithm.

The target distribution $\pi$ satisfies T1 if and only if there exist $a \in \X$ and $\beta > 0$ such that 
\begin{equation}
\label{eq:vil2210}
\int \exp(\beta \|x-a\|^2)d\pi(x) < \infty,
\end{equation}
see~\citet[Theorem 22.10]{villani2008optimal}. Therefore, \Cref{ass:T1} is essentially an assumption on the tails of $\pi$. In particular, $\pi \in \cP_2(\X)$.


\subsection{Reproducing Kernel Hilbert Space}

We consider a kernel $k$ associated to a Reproducing Kernel Hilbert Space (RKHS) denoted by $\cH_0$. We denote by $\Phi : \X \to \cH_0$ the so-called feature map $\Phi : x \mapsto k(\cdot,x)$. The product space $\cH_0^d$ is also a Hilbert space denoted $\cH \coloneqq \cH_0^d$. We make the following assumption on the kernel $k$.
\begin{assumption}
\label{ass:k_bounded}
There exists $ B>0$ such that the inequalities $$\|\Phi(x)\|_{\cH_0}\le B,$$
and
$$\|\nabla \Phi(x)\|_{\cH}^2 = 
 \sum_{i=1}^d \|\partial_{i} \Phi(x)\|^2_{\cH_0}\le B^2$$
hold for all $x \in \X$. Moreover, $\nabla \Phi : \X \to \cH$ is continuous.
\end{assumption}

\Cref{ass:k_bounded} is satisfied by the Gaussian kernel for example, with $B$ independent of $d$ using a scaling argument.
\Cref{ass:k_bounded} states that $\Phi : \X \to \cH_0$ is bounded, Lipschitz and $C^1$. This is satisfied by many classical kernels used in practice. Note that $k(x,x) =  \|\Phi(x)\|_{\cH_0}^2$ and that $\diver_1 \nabla_2 k(x,a) = \ps{\nabla \Phi(x),\nabla \Phi(a)}_{\cH}$ (in particular, $\diver_1 \nabla_2 k(x,x) = \|\nabla \Phi(x)\|_{\cH}^2$). Hence, $\nabla \Phi$ is continuous iff $x \mapsto \diver_1 \nabla_2 k(x,x)$ and $x \mapsto \diver_1 \nabla_2 k(x,a)$ are continuous for every $a \in \X$.

Under Assumption~\ref{ass:k_bounded}, $\cH \subset L^2(\mu)$ for every probability distribution on $\X$, and the inclusion map $\iota_{\mu}: \cH \to L^2(\mu)$ is continuous. We denote by $P_{\mu}: L^2(\mu) \to \cH$ its adjoint defined by the relation: for every $f \in L^2(\mu)$, $g \in \cH$, \begin{equation}
    \ps{f, \iota_{\mu} g}_{L^2(\mu)} = \ps{P_{\mu} f, g}_{\cH}.
\end{equation}
Then, $P_\mu$ can be expressed as a convolution with $k$~\citep[Proposition~3]{carmeli2010vector}:
\begin{equation}
    P_{\mu} f(x) = \int k(x,y)f(y)d\mu(y),
\end{equation}
or $P_{\mu} f = \int \Phi(y)f(y)d\mu(y)$ where the integral converges in norm.

\subsection{Stein Variational Gradient Descent}
\label{sec:svgd-desc}

\subsubsection{The population limit}

Stein Variational Gradient Descent (SVGD) is an algorithm to sample from $\pi \propto \exp(-F)$. SVGD proceeds by maintaining a set of $N$ particles over $\R^d$, whose empirical distribution $\mu_n^N$ at time $n$ aims to approximate $\pi$ as $n \to \infty$, see~\citet{liu2016stein}. The SVGD algorithm is presented above. 

\begin{algorithm}
    \caption{Stein Variational Gradient Descent~\citep{liu2016stein}}
    \begin{algorithmic}
        \STATE {\bf Initialization}: a set $x_0^1, \ldots, x_0^N \in \X$ of $N$ particles, a kernel $k$, a step size $\gamma > 0$. \FOR{$n=0,1,2,\ldots$}
        \FOR{$i=1,2,\ldots,N$}
        \STATE \begin{align*}
            \hspace{-0.8cm} x_{n+1}^i = {\color{blue} x_n^i} - \frac{\gamma}{N}\sum_{j = 1}^N k({\color{blue} x_n^i},x_n^j) \nabla F(x_n^j) - \nabla_2 k({\color{blue} x_n^i},x_n^j)
        \end{align*}
        \ENDFOR
        \ENDFOR
    \end{algorithmic}
    \label{tab:SVGD}
\end{algorithm}

 Denoting by $\mu_n^N$ the empirical distribution of $x_n^1,\ldots, x_n^N$, \textit{i.e.}, $$\mu_n^N \coloneqq \frac{1}{N}\sum_{i=1}^N \delta_{x_n^i}, $$ the SVGD update can be written
 \begin{align*}
            x_{n+1}^i =& {\color{blue} x_n^i} - \gamma \int k({\color{blue} x_n^i},y) \nabla F(y) - \nabla_2 k({\color{blue} x_n^i},y)d\mu_n^N(y)\\
            =& \left(I - \gamma \int k(\cdot,y) \nabla F(y) - \nabla_2 k(\cdot,y)d\mu_n^N(y)\right)({\color{blue} x_n^i}).
        \end{align*}
        Therefore, SVGD performs the update $$
        \mu_{n+1}^N =  \left(I - \gamma \int \Phi(y) \nabla F(y) - \nabla \Phi(y) d\mu_n^N(y)\right) \# \mu_{n}^N,
        $$
        at the level of measures. We call \textit{population limit} the regime where, formally, $N = \infty$. Mathematically, this corresponds to the assumption that $\mu_0$ has a density (which can be seen as intuitively seen $\bE \lim_{N \to \infty} \mu_0^N$) which belongs to $C_0(\X)$. In this case, we shall see in our analysis that $\mu_n$ has a density for every $n$. To summarize, in the population limit, SVGD performs the same update:
\begin{equation}
\label{eq:svgdpopulation}
    \mu_{n+1} = \left(I - \gamma h_{\mu_n}\right)\#\mu_n,
\end{equation}
where
\begin{equation*}
    h_{\mu}(x) \coloneqq \int k(x,y)\nabla F(y) - \nabla_y k(x,y) d\mu(y)\end{equation*}
    or 
    \begin{equation*} h_{\mu} \coloneqq \int \Phi(y)\nabla F(y)  - \nabla \Phi(y) d\mu(y),
\end{equation*}
and where $\mu_n$ has a density.

Finally, note that the SVGD algorithm was originally derived in~\citet{liu2016stein} from its population limit. The authors first introduced the SVGD update in the population limit, and then, the SVGD algorithm (Algorithm~\ref{tab:SVGD}) is obtained from the population limit by approximating the expectations by empirical means. 

\paragraph{Our point of view on SVGD in the population limit.} We now provide the intuition behind our results on SVGD.

\begin{quote}\em In the population limit, SVGD can be seen as a Riemannian gradient descent, thanks to the following two reasons. \end{quote}

First, in a Riemannian interpretation of the Wasserstein space~\citep{villani2008optimal}, for every $\mu \in \cP_2(\X)$, the map $\exp_{\mu} : \phi \mapsto (I + \phi)\#\mu$ can be seen as the exponential map at $\mu$. In the population limit, SVGD~\eqref{eq:svgdpopulation} can be rewritten as $$\mu_{n+1} = \exp_{\mu_n}(-\gamma h_{\mu_n}).$$ 

Second, $-h_{\mu}$ can be seen as the negative gradient of $\cF$ at $\mu$ under a certain metric. Indeed, using  integration by parts, $h_{\mu} = P_{\mu} \nabla_W \cF(\mu)$, see \textit{e.g.}~\citet{korba2020non,duncan2019geometry}. Therefore, for every $g \in \cH$, $\ps{h_{\mu},g}_{\cH} = \ps{\nabla_W \cF(\mu),g}_{L^2(\mu)}$, hence $h_{\mu}$ can be seen as a Wasserstein gradient of $\cF$ under the inner product of $\cH$.

The Kernelized Stein Discrepancy (KSD) is a natural discrepancy between probability distributions that was introduced prior to SVGD~\citep{liu2016kernelized,chwialkowski2016kernel} to compare probablity measures. Indeed, if the RKHS $\cH$ is rich enough~\citep{liu2016kernelized,chwialkowski2016kernel,oates2019convergence}, an assumption that we shall always make in this paper, then 
\begin{equation*}
\KSD(\mu|\pi) = 0 \Longrightarrow \mu = \pi.    
\end{equation*}
 The KSD is intimately related to SVGD, and the KSD naturally appears in the original derivation of SVGD~\citep{liu2016stein}. The KSD is defined as the square root of the Stein Fisher Information~\citep{duncan2019geometry} $I_{\st}$: 
\begin{equation}
\label{eq:Istein}
    I_{\st}(\mu|\pi) \coloneqq \|h_{\mu}\|_{\cH}^2, \quad \KSD(\mu|\pi) \coloneqq \|h_{\mu}\|_{\cH}^2. 
\end{equation}
In this paper, we study the complexity of SVGD in terms of the KSD. To understand better the topology of the KSD and compare it to common topologies in the space of probability measures, we refer to~\citet{gorham2017measuring}. 

\section{Analysis of SVGD}
\label{sec:svgd}

In this section, we analyze SVGD in the infinite number of particles regime. Recall that in this regime, SVGD is given by $\mu_0 \in C_0(\X)$ and
\begin{equation*}
    \mu_{n+1} = (I - \gamma h_{\mu_n}) \# \mu_n,
\end{equation*}
where 
\begin{equation*}
    h_{\mu} \coloneqq \int \nabla F(x) \Phi(x) - \nabla \Phi(x) d\mu(x).
\end{equation*}

\subsection{A fundamental inequality}

We start by stating a fundamental inequality satisfied by $\cF$ for any update of the form
\begin{equation}
    \mu_{n+1} = \left(I - \gamma g\right) \# \mu_n,
\end{equation}
where $g \in \cH$.

\begin{proposition}
\label{prop:TL}
Let Assumptions~\ref{ass:V_Lipschitz} and~\ref{ass:k_bounded} hold true. Let $\alpha > 1$ and choose $\gamma > 0$ such that $\gamma \|g\|_{\cH} \leq \frac{\alpha-1}{\alpha B}$. Then,
    \begin{equation}
    \label{eq:TL}
       \cF(\mu_{n+1}) \leq \cF(\mu_{n}) - \gamma\ps{h_{\mu_n},g}_{\cH} +  \frac{\gamma^2 K}{2}\|g\|_{\cH}^2,
    \end{equation}
    where $K = (\alpha^2 + L)B$.
\end{proposition}
Inequality~\eqref{eq:TL} is a property of \textit{the functional $\cF$}, and not a property of the SVGD algorithm. Inequality~\eqref{eq:TL} plays the role of a \textit{Taylor inequality} for the functional $\cF$, where $h_{\mu_n}$ is the Wasserstein gradient of $\cF$ at $\mu_n$ under the metric induced by $\cH$. Proposition~\ref{prop:TL} is a slight generalization of \citet[Proposition~5]{korba2020non}, and is not our main contribution, therefore we only sketch its proof in the Appendix. 


\subsection{Main result}
Applying recursively the Taylor inequality Proposition~\ref{prop:TL} with $g = h_{\mu_n}$, we obtain the following descent property for SVGD, which is our main theoretical result. The proof of this result can be found in the Appendix.

\begin{theorem}[Descent lemma]
\label{th:svgd}
  Let Assumptions \ref{ass:V_Lipschitz}, \ref{ass:T1} and \ref{ass:k_bounded}  hold true. Let $\alpha > 1$.   If 
  \begin{align}
  \label{eq:condition-step}
      &\gamma \leq (\alpha-1)\times \\
      &\left(\alpha B^2 \left(1 + \|\nabla F(0)\| + L \int \|x\|d\pi(x) + L \sqrt{\frac{2\cF(\mu_0)}{\lambda}} \right)\right)^{-1}\nonumber
  \end{align}
or 
  \begin{align}
  \label{eq:condition-step-2}
      &\gamma \leq (\alpha-1) \times\\
      &\left(\alpha B^2 \left(1 + 2L\sqrt{\frac{2\cF(\mu_0)}{\lambda}} + L \int \|x-x_\star\|d\mu_0(x))\right)\right)^{-1} \nonumber
  \end{align}
  then
  \begin{equation}
    \label{eq:TL-svgd-cst}
        \cF(\mu_{n+1}) \leq \cF(\mu_{n}) - \gamma\left(1 - \frac{\gamma B (\alpha^2 + L)}{2}\right)\KSD^2(\mu_n|\pi).
    \end{equation}
\end{theorem}

If $\cF(\mu_0) < \infty$, then, using Theorem~\ref{th:svgd}, $(\cF(\mu_n))_n$ is nonincreasing and $\mu_n$ has a density w.r.t. Lebesgue measure for every $n$ (since $\cF(\mu_n) < \infty$).

In the language of the gradient descent algorithms, Theorem~\ref{th:svgd} is called a descent property. It can be seen as a discrete time analogue of dissipation properties obtained for the PDE modeling SVGD in continuous time in the population limit~\citep{duncan2019geometry,korba2020non}.

Unlike~\citet[Proposition 5]{korba2020non} and~\citet[Theorem 3.3]{liu2017stein}, we do not assume that $\sup_n \KSD(\mu_n|\pi) < \infty$ or that $\gamma \leq \KSD(\mu_n|\pi)^{-1}$ to obtain our descent property. The step size $\gamma$ is bounded by a constant. Iterating Theorem~\ref{th:svgd}, we obtain convergence results as corollaries in the next section.


\section{Convergence and Complexity}
\label{sec:weak}

\subsection{Convergence}
We now show that Theorem~\ref{th:svgd} implies  weak convergence and convergence in $W_1$. 


\begin{corollary}[Weak convergence]
Let Assumptions \ref{ass:V_Lipschitz}, \ref{ass:T1} and \ref{ass:k_bounded} hold true. Let $\alpha > 1$. 
  If $\gamma < \frac{2}{B (\alpha^2 + L)}$, and $\gamma $ further satisfies either \eqref{eq:condition-step} or \eqref{eq:condition-step-2},  then $\mu_n \rightarrow_{n \to +\infty} \pi$ weakly and $W_1(\mu_n,\pi) \to 0$.
\end{corollary}

\begin{proof}
    Using Theorem~\ref{th:svgd} and iterating, 
    \begin{equation*}
        \cF(\mu_{n}) \leq \cF(\mu_{0}) - \gamma\left(1 - \frac{\gamma B (\alpha^2 + L)}{2}\right) \sum_{k=0}^{n-1} \KSD^2(\mu_k|\pi).
    \end{equation*}
    Therefore, $\cF(\mu_n)$ is uniformly bounded. 
    For every $n \geq 1$,
    \begin{equation*}
    \gamma\left(1 - \frac{\gamma B (\alpha^2 + L)}{2}\right) \sum_{k=0}^{n-1} \KSD^2(\mu_k|\pi) \leq \cF(\mu_0).
    \end{equation*}
    Consequently, $\sum_{n=0}^{+\infty} \KSD^2(\mu_n|\pi) < \infty$. Therefore $\KSD(\mu_n|\pi) \rightarrow_{n \to +\infty} 0$.

    Moreover, using~\Cref{ass:T1} and~\eqref{eq:vil2210}, for every $a \in \X$, $\int \exp(\ps{a,x})d\pi(x) < \infty$. Therefore, using~\citet[Lemma 1.4.3]{dupuis2011weak}, $(\mu_n)$ is both tight and uniformly integrable. Consider a subsequence of $(\mu_{\phi(n)})$ converging weakly to some $\mu_\star$. We shall prove that $\mu_\star = \pi$.
    
    First, using~\Cref{ass:V_Lipschitz} and~\Cref{ass:k_bounded}, $x \mapsto \nabla F(x)\Phi(x) - \nabla \Phi(x) \in \cH$ is continuous and \begin{align}
    \label{eq:sublin}
        &\left\| \nabla F(x)\Phi(x) - \nabla \Phi(x) \right\|_{\cH} \nonumber\\
        &\left\| \nabla F(x)\Phi(x)\right\|_{\cH} + \left\|\nabla \Phi(x) \right\|_{\cH} \nonumber\\
        =& \left\|\nabla F(x)\right\|\left\|\Phi(x)\right\|_{\cH_0} + \left\|\nabla \Phi(x) \right\|_{\cH} \nonumber\\
        \leq& B \left( \left\|\nabla F(x)\right\| + 1 \right) \nonumber\\
        \leq& B \left( \left\|\nabla F(0)\right\| + L \|x\|+ 1 \right).\nonumber
    \end{align}
    Moreover, as a subsequence, $(\mu_{\phi(n)})$ is also uniformly integrable and also converges weakly to $\mu_\star$. Therefore, using~\citet[Theorem 7.12]{villani2003topics} with $p = 1$,
        $\E_{x \sim \mu_{\phi(n)}} \left(\nabla F(x)\Phi(x) - \nabla \Phi(x)\right)$ converges to $\E_{x \sim \mu_\star} \left(\nabla F(x)\Phi(x) - \nabla \Phi(x)\right)$
    in $\cH$. In other words, $h_{\mu_{\phi(n)}}$ converges to $h_{\mu_\star}$ in $\cH$. Taking the norm, $\KSD(\mu_{\phi(n)}|\pi) \rightarrow \KSD(\mu_\star|\pi)$ along the subsequence. Recalling that $\KSD(\mu_n | \pi) \rightarrow 0$ we obtain $\KSD(\mu_\star|\pi) = 0$, which implies $\mu_\star = \pi$. 
    
    In conclusion, $\mu_n \rightarrow_{n \to +\infty} \pi$ weakly. Moreover, the convergence also happens in $W_1$ because $(\mu_n)$ is uniformly integrable, see~\citep[Theorem 7.12]{villani2003topics}.
\end{proof}

In summary, under T1 and some smoothness assumptions but \textit{without convexity of the potential}, SVGD in the population limit converges to the target distribution.

One can be surprised to see that SVGD converges without convexity assumption on $F$, but this is actually natural if one thinks about the gradient descent interpretation of SVGD. Indeed, SVGD in the population limit is a gradient descent on the KL divergence, which is 
\begin{itemize}
    \item "smooth" if we restrict the descent directions to a RKHS (\textit{i.e.}, it satisfies a Taylor inequality Proposition~\ref{prop:TL}),
    \item coercive (\textit{i.e.}, sublevel sets are tight)~\citet[Lemma 1.4.3]{dupuis2011weak},
    \item and has a single stationary point which is its global minimizer (the KSD is the norm of the gradient of KL in our interpretation, and the KSD is equal to zero only at the optimum).
\end{itemize}  
One can show that, over $\R^d$, gradient descent applied to a smooth coercive function with a single stationary point converges to the global minimizer. The situation here is similar.



\subsection{Complexity}

Next, we provide a $\cO(1/n)$ convergence rate for the empirical mean of the iterates $\mu_n$ in terms of the squared KSD. This result is obtained from our descent lemma (Theorem~\ref{th:svgd}).

\begin{corollary}[Convergence rate]
\label{cor:conv}
Let Assumptions \ref{ass:V_Lipschitz}, \ref{ass:T1} and  \ref{ass:k_bounded} hold true. Let $\alpha > 1$. 
  If $\gamma < \frac{2}{B (\alpha^2 + L)}$, and $\gamma $ further satisfies either \eqref{eq:condition-step} or \eqref{eq:condition-step-2},   then
  \begin{equation}
     I_{\st}(\bar{\mu}_n|\pi) \leq \frac{2\cF(\mu_0)}{n \gamma},
  \end{equation}
  where $\bar{\mu}_n = \frac{1}{n} \sum_{k=0}^{n-1} \mu_k$.
\end{corollary}
Note that this convergence rate is given in terms of the uniform mixture of $\mu_0, \ldots, \mu_{n-1}$. Similar mixtures appear in the analysis of Langevin algorithm (see \textit{e.g.}~\citet{durmus2018analysis}). Note also that the convergence rate in Corollary~\ref{cor:conv} is similar to the convergence rate of the squared norm of the gradient in the gradient descent algorithm applied to a smooth function~\citep{nesterov2013introductory}.
\begin{proof}
    Using Theorem~\ref{th:svgd}, 
$
        \cF(\mu_{n+1}) \leq \cF(\mu_{n}) - \frac{\gamma}{2} \KSD^2(\mu_n|\pi),
$
    and by iterating, we get
    \begin{equation*}
        0 \leq \cF(\mu_{n}) \leq \cF(\mu_{0}) - \frac{\gamma}{2} \sum_{k=0}^{n-1} \|h_{\mu_k}\|^2.
    \end{equation*}
    Rearranging the terms, and using the convexity of the squared norm,
    \begin{equation*}
        \|h_{\bar{\mu}_n}\|^2 = \left \|\frac{1}{n}\sum_{k=0}^{n-1} h_{\mu_k} \right\|^2 \leq \frac{1}{n}\sum_{k=0}^{n-1} \|h_{\mu_k}\|^2 \leq \frac{2 \cF(\mu_0)}{n\gamma}.
    \end{equation*}
\end{proof}


From the last result, we can characterize the iteration complexity of SVGD.
\begin{corollary}[Complexity]
\label{cor:comp}
Let Assumptions \ref{ass:V_Lipschitz}, \ref{ass:T1} and \ref{ass:k_bounded} hold true. Let $\alpha > 1$. If $\gamma \leq \min(\frac{2}{B (\alpha^2 + L)},\frac{\alpha-1}{\alpha K})$, where
\begin{equation*}
      K \coloneqq B^2 \left(1 + 2L\sqrt{\frac{2}{\lambda}}\sqrt{F(x_\star) + \frac{d}{2}\log \left(\frac{L}{2\pi}\right)} + \sqrt{L d}\right),
  \end{equation*}
  and if $\mu_0 = \cN(x_\star,\frac{1}{L} I)$,
  then $$n = \tilde{\Omega}\left(\frac{L{d}^{3/2}}{\lambda^{1/2}\varepsilon}\right)$$ iterations of SVGD suffice to output $\mu \coloneqq \bar{\mu}_n$ such that $I_{\st}(\mu|\pi) \leq \varepsilon.$
\end{corollary}
To our knowledge, Corollary~\ref{cor:comp} provides the first \textit{dimension-dependent} complexity result for SVGD. Its proof can be found in the appendix. The dependence of the T1 constant $\lambda$ in the dimension $d$ is subject to active research in optimal transport theory~\citep[Remark 22.11]{villani2008optimal} and is out of the scope of this paper. Yet, using~\citet[Theorem 22.10, Equation 22.16]{villani2008optimal}, $\lambda$ can be taken as
\begin{equation*}
      1/\lambda = \min_{a \in \mathcal X,\beta > 0}\frac{1}{\beta^2}\left(1 + \log \int \exp(\beta \|x-a\|^2)d\pi(x)\right).
\end{equation*}

Note that the output $\mu$ of the algorithm is a mixture of the iterates: $\mu = \bar{\mu}_n$. Besides, optimizing the complexity over $\alpha$ leads to involved calculations that do not change the overall complexity. To see this, note that the larger the step size $\gamma$, the smaller the complexity. But, even if the step size $\gamma = \min(\frac{2}{BL},1/K)$ were allowed, the overall complexity would be the same.
\section{Conclusion}

We proved that under T1 inequality and some smoothness assumptions on the kernel and the potential of the target distribution but without convexity, SVGD in the population limit converges weakly and in 1-Wasserstein distance to the target distribution. Moreover, we showed that SVGD reaches $\varepsilon$ accuracy in terms of the squared Kernelized Stein Discrepancy after $\tilde{\Omega} \left(\frac{{d}^{3/2}}{\varepsilon}\right)$ iterations.

A possible extension of our work is to study SVGD under functional inequalities other than T1, such as \citep[Corollary 2.6 (i)]{bolley2005weighted} (which is weaker than T1) or the Poincaré inequality (in the form of~\citet[Theorem 22.25 (iii)]{villani2008optimal}). We claim that our approach can be extended to study SVGD in these settings. 


Finally, an important and difficult open problem in the analysis of SVGD is to characterize its complexity with a finite number of particles, i.e. with discrete measures. In this regime, we lose the interpretation of SVGD as a gradient descent in the space of probability measures, because the KL divergence w.r.t. the target distribution is infinite. However, we believe that our clean analysis in the population limit makes a first step towards this open problem.   

\newpage
\bibliography{math}

\newcommand{\noop}[1]{} \def\cprime{$'$} \def\cdprime{$''$} \def\cprime{$'$}
\begin{thebibliography}{44}
\providecommand{\natexlab}[1]{#1}
\providecommand{\url}[1]{\texttt{#1}}
\expandafter\ifx\csname urlstyle\endcsname\relax
  \providecommand{\doi}[1]{doi: #1}\else
  \providecommand{\doi}{doi: \begingroup \urlstyle{rm}\Url}\fi

\bibitem[Balasubramanian et~al.(2022)Balasubramanian, Chewi, Erdogdu, Salim,
  and Zhang]{balasubramanian2022towards}
Balasubramanian, K., Chewi, S., Erdogdu, M.~A., Salim, A., and Zhang, M.
\newblock Towards a theory of non-log-concave sampling: first-order
  stationarity guarantees for langevin monte carlo.
\newblock \emph{arXiv preprint arXiv:2202.05214}, 2022.

\bibitem[Bernton(2018)]{bernton2018langevin}
Bernton, E.
\newblock {L}angevin {M}onte {C}arlo and {JKO} splitting.
\newblock In \emph{Conference on Learning Theory (COLT)}, pp.\  1777--1798,
  2018.

\bibitem[Bolley \& Villani(2005)Bolley and Villani]{bolley2005weighted}
Bolley, F. and Villani, C.
\newblock Weighted csisz{\'a}r-kullback-pinsker inequalities and applications
  to transportation inequalities.
\newblock In \emph{Annales de la Facult{\'e} des sciences de Toulouse:
  Math{\'e}matiques}, volume~14, pp.\  331--352, 2005.

\bibitem[Bubeck et~al.(2018)Bubeck, Eldan, and Lehec]{bubeck2018sampling}
Bubeck, S., Eldan, R., and Lehec, J.
\newblock Sampling from a log-concave distribution with projected {L}angevin
  {M}onte {C}arlo.
\newblock \emph{Discrete \& Computational Geometry}, 59\penalty0 (4):\penalty0
  757--783, 2018.

\bibitem[Carmeli et~al.(2010)Carmeli, De~Vito, Toigo, and
  Umanit{\'a}]{carmeli2010vector}
Carmeli, C., De~Vito, E., Toigo, A., and Umanit{\'a}, V.
\newblock Vector valued reproducing kernel {H}ilbert spaces and universality.
\newblock \emph{Analysis and Applications}, 8\penalty0 (01):\penalty0 19--61,
  2010.

\bibitem[Cheng et~al.(2018)Cheng, Chatterji, Bartlett, and
  Jordan]{cheng2017underdamped}
Cheng, X., Chatterji, N.~S., Bartlett, P.~L., and Jordan, M.~I.
\newblock Underdamped {L}angevin {MCMC}: {A} non-asymptotic analysis.
\newblock In \emph{Conference on Learning Theory (COLT)}, pp.\  300--323, 2018.

\bibitem[Chewi et~al.(2020)Chewi, Gouic, Lu, Maunu, and
  Rigollet]{chewi2020svgd}
Chewi, S., Gouic, T.~L., Lu, C., Maunu, T., and Rigollet, P.
\newblock {SVGD} as a kernelized {W}asserstein gradient flow of the chi-squared
  divergence.
\newblock In \emph{Advances in Neural Information Processing Systems
  (NeurIPS)}, pp.\  2098--2109, 2020.

\bibitem[Chwialkowski et~al.(2016)Chwialkowski, Strathmann, and
  Gretton]{chwialkowski2016kernel}
Chwialkowski, K., Strathmann, H., and Gretton, A.
\newblock A kernel test of goodness of fit.
\newblock In \emph{International Conference on Machine Learning (ICML)}, pp.\
  2606--2615, 2016.

\bibitem[Dalalyan(2017)]{dalalyan2017theoretical}
Dalalyan, A.~S.
\newblock Theoretical guarantees for approximate sampling from smooth and
  log-concave densities.
\newblock \emph{Journal of the Royal Statistical Society: Series B (Statistical
  Methodology)}, 79\penalty0 (3):\penalty0 651--676, 2017.

\bibitem[Duncan et~al.(2019)Duncan, N{\"u}sken, and
  Szpruch]{duncan2019geometry}
Duncan, A., N{\"u}sken, N., and Szpruch, L.
\newblock On the geometry of {S}tein variational gradient descent.
\newblock \emph{arXiv preprint arXiv:1912.00894}, 2019.

\bibitem[Dupuis \& Ellis(2011)Dupuis and Ellis]{dupuis2011weak}
Dupuis, P. and Ellis, R.~S.
\newblock \emph{A weak convergence approach to the theory of large deviations},
  volume 902.
\newblock John Wiley \& Sons, 2011.

\bibitem[Durmus \& Moulines(2017)Durmus and Moulines]{durmus2017nonasymptotic}
Durmus, A. and Moulines, E.
\newblock Nonasymptotic convergence analysis for the unadjusted {L}angevin
  algorithm.
\newblock \emph{The Annals of Applied Probability}, 27\penalty0 (3):\penalty0
  1551--1587, 2017.

\bibitem[Durmus et~al.(2018)Durmus, Moulines, and Pereyra]{durmus2018efficient}
Durmus, A., Moulines, E., and Pereyra, M.
\newblock Efficient {B}ayesian computation by proximal {M}arkov {C}hain {M}onte
  {C}arlo: when {L}angevin meets {M}oreau.
\newblock \emph{SIAM Journal on Imaging Sciences}, 11\penalty0 (1):\penalty0
  473--506, 2018.

\bibitem[Durmus et~al.(2019)Durmus, Majewski, and
  Miasojedow]{durmus2018analysis}
Durmus, A., Majewski, S., and Miasojedow, B.
\newblock Analysis of {L}angevin {M}onte {C}arlo via convex optimization.
\newblock \emph{Journal of Machine Learning Research}, 20\penalty0
  (73):\penalty0 1--46, 2019.

\bibitem[Foster et~al.(2021)Foster, Lyons, and Oberhauser]{foster2021shifted}
Foster, J., Lyons, T., and Oberhauser, H.
\newblock The shifted ode method for underdamped langevin mcmc.
\newblock \emph{arXiv preprint arXiv:2101.03446}, 2021.

\bibitem[Gorham \& Mackey(2017)Gorham and Mackey]{gorham2017measuring}
Gorham, J. and Mackey, L.
\newblock Measuring sample quality with kernels.
\newblock In \emph{International Conference on Machine Learning (ICML)}, pp.\
  1292--1301, 2017.

\bibitem[Gorham et~al.(2020)Gorham, Raj, and Mackey]{gorham2020stochastic}
Gorham, J., Raj, A., and Mackey, L.
\newblock Stochastic stein discrepancies.
\newblock \emph{Advances in Neural Information Processing Systems (NeurIPS)},
  33:\penalty0 17931--17942, 2020.

\bibitem[Hsieh et~al.(2018)Hsieh, Kavis, Rolland, and
  Cevher]{hsieh2018mirrored}
Hsieh, Y.-P., Kavis, A., Rolland, P., and Cevher, V.
\newblock Mirrored {L}angevin dynamics.
\newblock In \emph{Advances in Neural Information Processing Systems
  (NeurIPS)}, pp.\  2878--2887, 2018.

\bibitem[Kassab \& Simeone(2020)Kassab and Simeone]{kassab2020federated}
Kassab, R. and Simeone, O.
\newblock Federated generalized bayesian learning via distributed stein
  variational gradient descent.
\newblock \emph{arXiv preprint arXiv:2009.06419}, 2020.

\bibitem[Korba et~al.(2020)Korba, Salim, Arbel, Luise, and
  Gretton]{korba2020non}
Korba, A., Salim, A., Arbel, M., Luise, G., and Gretton, A.
\newblock A non-asymptotic analysis for {S}tein variational gradient descent.
\newblock In \emph{Advances in Neural Information Processing Systems
  (NeurIPS)}, pp.\  4672--4682, 2020.

\bibitem[Li et~al.(2021)Li, Zha, and Tao]{li2021sqrt}
Li, R., Zha, H., and Tao, M.
\newblock Sqrt (d) dimension dependence of langevin monte carlo.
\newblock \emph{arXiv preprint arXiv:2109.03839}, 2021.

\bibitem[Liu(2017)]{liu2017stein}
Liu, Q.
\newblock Stein variational gradient descent as gradient flow.
\newblock In \emph{Advances in Neural Information Processing Systems (NIPS)},
  pp.\  3115--3123, 2017.

\bibitem[Liu \& Wang(2016)Liu and Wang]{liu2016stein}
Liu, Q. and Wang, D.
\newblock Stein variational gradient descent: A general purpose {B}ayesian
  inference algorithm.
\newblock In \emph{Advances in Neural Information Processing Systems (NIPS)},
  pp.\  2378--2386, 2016.

\bibitem[Liu et~al.(2016)Liu, Lee, and Jordan]{liu2016kernelized}
Liu, Q., Lee, J., and Jordan, M.
\newblock A kernelized {S}tein discrepancy for goodness-of-fit tests.
\newblock In \emph{International Conference on Machine Learning (ICML)}, pp.\
  276--284, 2016.

\bibitem[Liu et~al.(2017)Liu, Ramachandran, Liu, and Peng]{liu2017policy}
Liu, Y., Ramachandran, P., Liu, Q., and Peng, J.
\newblock Stein variational policy gradient.
\newblock \emph{arXiv preprint arXiv:1704.02399}, 2017.

\bibitem[Lu et~al.(2019)Lu, Lu, and Nolen]{lu2019scaling}
Lu, J., Lu, Y., and Nolen, J.
\newblock Scaling limit of the {S}tein variational gradient descent: The mean
  field regime.
\newblock \emph{SIAM Journal on Mathematical Analysis}, 51\penalty0
  (2):\penalty0 648--671, 2019.

\bibitem[Ma et~al.(2019)Ma, Chatterji, Cheng, Flammarion, Bartlett, and
  Jordan]{ma2019there}
Ma, Y.-A., Chatterji, N., Cheng, X., Flammarion, N., Bartlett, P.~L., and
  Jordan, M.~I.
\newblock Is there an analog of {N}esterov acceleration for {MCMC}?
\newblock \emph{arXiv preprint arXiv:1902.00996}, 2019.

\bibitem[Nesterov(2013)]{nesterov2013introductory}
Nesterov, Y.~E.
\newblock \emph{Introductory lectures on convex optimization: A basic course},
  volume~87.
\newblock Springer Science \& Business Media, 2013.

\bibitem[N{\"u}sken \& Renger(2021)N{\"u}sken and Renger]{nusken2021stein}
N{\"u}sken, N. and Renger, D.
\newblock Stein variational gradient descent: many-particle and long-time
  asymptotics.
\newblock \emph{arXiv preprint arXiv:2102.12956}, 2021.

\bibitem[Oates et~al.(2019)Oates, Cockayne, Briol, and
  Girolami]{oates2019convergence}
Oates, C.~J., Cockayne, J., Briol, F.-X., and Girolami, M.
\newblock Convergence rates for a class of estimators based on stein’s
  method.
\newblock \emph{Bernoulli}, 25\penalty0 (2):\penalty0 1141--1159, 2019.

\bibitem[Pu et~al.(2017)Pu, Gan, Henao, Li, Han, and Carin]{pu2017vae}
Pu, Y., Gan, Z., Henao, R., Li, C., Han, S., and Carin, L.
\newblock {VAE} learning via {S}tein variational gradient descent.
\newblock In \emph{Advances in Neural Information Processing Systems (NIPS)},
  pp.\  4236--4245, 2017.

\bibitem[Rolland et~al.(2020)Rolland, Eftekhari, Kavis, and
  Cevher]{rolland2020double}
Rolland, P., Eftekhari, A., Kavis, A., and Cevher, V.
\newblock Double-loop unadjusted {L}angevin algorithm.
\newblock In \emph{International Conference on Machine Learning (ICML)}, pp.\
  8169--8177, 2020.

\bibitem[Salim \& Richt{\'a}rik(2020)Salim and Richt{\'a}rik]{salim2020primal}
Salim, A. and Richt{\'a}rik, P.
\newblock Primal dual interpretation of the proximal stochastic gradient
  {L}angevin algorithm.
\newblock In \emph{Advances in Neural Information Processing Systems
  (NeurIPS)}, pp.\  3786--3796, 2020.

\bibitem[Shen \& Lee(2019)Shen and Lee]{shen2019randomized}
Shen, R. and Lee, Y.~T.
\newblock The randomized midpoint method for log-concave sampling.
\newblock In \emph{Advances in Neural Information Processing Systems
  (NeurIPS)}, pp.\  2100--2111, 2019.

\bibitem[Shi et~al.(2021)Shi, Liu, and Mackey]{shi2021sampling}
Shi, J., Liu, C., and Mackey, L.
\newblock Sampling with mirrored stein operators.
\newblock \emph{arXiv preprint arXiv:2106.12506}, 2021.

\bibitem[{\c{S}}im{\c{s}}ekli(2017)]{csimvsekli2017fractional}
{\c{S}}im{\c{s}}ekli, U.
\newblock Fractional {L}angevin {M}onte {C}arlo: Exploring {L}{\'e}vy driven
  stochastic differential equations for {M}arkov {C}hain {M}onte {C}arlo.
\newblock In \emph{International Conference on Machine Learning (ICML)}, pp.\
  3200--3209, 2017.

\bibitem[Tao et~al.(2019)Tao, Dai, Chen, Bai, Chen, Liu, Zhang, Bobashev, and
  Carin]{tao2019variational}
Tao, C., Dai, S., Chen, L., Bai, K., Chen, J., Liu, C., Zhang, R., Bobashev,
  G., and Carin, L.
\newblock Variational annealing of {GANs}: A {L}angevin perspective.
\newblock In \emph{International Conference on Machine Learning (ICML)}, pp.\
  6176--6185, 2019.

\bibitem[Vempala \& Wibisono(2019)Vempala and Wibisono]{vempala2019rapid}
Vempala, S. and Wibisono, A.
\newblock Rapid convergence of the unadjusted {L}angevin algorithm:
  Isoperimetry suffices.
\newblock In \emph{Advances in Neural Information Processing Systems
  (NeurIPS)}, pp.\  8092--8104, 2019.

\bibitem[Villani(2003)]{villani2003topics}
Villani, C.
\newblock \emph{Topics in optimal transportation}.
\newblock Number~58 in Graduate Studies in Mathematics. American Mathematical
  Society, 2003.

\bibitem[Villani(2008)]{villani2008optimal}
Villani, C.
\newblock \emph{Optimal transport: old and new}, volume 338.
\newblock Springer Science \& Business Media, 2008.

\bibitem[Wibisono(2018)]{wibisono2018sampling}
Wibisono, A.
\newblock Sampling as optimization in the space of measures: The {L}angevin
  dynamics as a composite optimization problem.
\newblock In \emph{Conference on Learning Theory (COLT)}, pp.\  2093–3027,
  2018.

\bibitem[Zhang et~al.(2018)Zhang, Li, Chen, and Carin]{zhang2018learning}
Zhang, R., Li, C., Chen, C., and Carin, L.
\newblock Learning structural weight uncertainty for sequential
  decision-making.
\newblock In \emph{International Conference on Artificial Intelligence and
  Statistics (AISTATS)}, pp.\  1137--1146, 2018.

\bibitem[Zhang et~al.(2019)Zhang, Wen, Chen, Fang, Yu, and
  Carin]{zhang2019scalable}
Zhang, R., Wen, Z., Chen, C., Fang, C., Yu, T., and Carin, L.
\newblock Scalable {T}hompson sampling via optimal transport.
\newblock In \emph{International Conference on Artificial Intelligence and
  Statistics (AISTATS)}, pp.\  87--96, 2019.

\bibitem[Zou et~al.(2019)Zou, Xu, and Gu]{zou2019sampling}
Zou, D., Xu, P., and Gu, Q.
\newblock Sampling from non-log-concave distributions via variance-reduced
  gradient {L}angevin dynamics.
\newblock In \emph{International Conference on Artificial Intelligence and
  Statistics (AISTATS)}, pp.\  2936--2945, 2019.

\end{thebibliography}
\bibliographystyle{icml2022}

\newpage
\appendix
\onecolumn

\part*{Appendix}

\section{Proof of Proposition~\ref{prop:stationary}}

First, we prove that $F$ is coercive, \textit{i.e.}, for every $C>0$, the set $S = \{x \in \X : F(x) \leq C\}$ is compact. Since $F$ is continuous, $S$ is closed. It remains to prove that $S$ is bounded. Assume, by contradiction, that $S$ is unbounded. Then, there exists a sequence $(x_n)$ of points in $\X$ such that $F(x_n) \leq C$, $\|x_n\| \to +\infty$ and $B(x_n) \cap B(x_m) = \emptyset$ for every $n \neq m$, where $B(x)$ denotes  the unit ball centered at $x$. 

Let $n \geq 0$. Using the smoothness of $F$ (\Cref{ass:V_Lipschitz}), for every $x \in B(x_n)$,
\begin{equation*}
    F(x) \leq F(x_n) + \ps{\nabla F(x_n),x-x_n} + \frac{L}{2}.
\end{equation*}
Denote by $V$ the volume of the unit ball centered at $x$, \textit{i.e.} its Lebesgue measure. The positive number $V$ does not depend on $x$. Then

\begin{eqnarray*}
  \int_{B(x_n)} \exp(-F(x))dx & \geq & \int_{B(x_n)} \exp \left(-F(x_n) -\ps{\nabla F(x_n),x-x_n} - \frac{L}{2}\right)dx \\
   & =& V \exp \left(-F(x_n)-\frac{L}{2}\right) \int_{B(x_n)} \exp\left(\ps{\nabla F(x_n),x_n-x}\right)\frac{dx}{V}\\
    &=& V \exp \left(-F(x_n)-\frac{L}{2}\right) \int_{B(0)} \exp\left(\ps{\nabla F(x_n),u}\right)\frac{du}{V}\\
   & \geq& V \exp \left(-F(x_n)-\frac{L}{2}\right)  \exp\left( \int_{B(0)} \ps{\nabla F(x_n),u} \frac{du}{V}\right)\\
   & =& V \exp \left(-F(x_n)-\frac{L}{2}\right)\\
   & \geq& V \exp \left(-C-\frac{L}{2}\right),
\end{eqnarray*}
where we used Jensen's inequality for the uniform distribution over $B(0)$, thanks to the convexity of $t\mapsto \exp(t)$. Finally, 
\begin{equation*}
\int \exp(-F(x))dx \geq \sum_{n = 0}^{\infty} \int_{B(x_n)} \exp(-F(x))dx \geq \sum_{n = 0}^{\infty} V \exp \left(-C-\frac{L}{2}\right) = +\infty,
\end{equation*}
which means that $\exp(-F)$ is not integrable. This contradicts the definition of $F$ and therefore, $S$ is bounded.

Next, since the set $S$ is compact,  $F$ is coercive, and hence $F$ admits a stationary point. Indeed, $F$ is continuous over the compact set $\{x \in \X : F(x) \leq 1\}$, and therefore, $F$ admits a minimizer, $x_\star$, over this set. Moreover, this point $x_\star$ is a stationary point \textit{i.e.}, $\nabla F(x_\star) = 0$ (note that the point $x_\star$ is actually a global minimizer of $F$).

\section{Proof of Proposition~\ref{prop:TL}}

    Let $\phi_t = I - t g$ for $t \in [0,\gamma]$ and $\rho_t = (\phi_t) \# \mu_n$. Note that $\rho_0 = \mu_n$ and $\rho_{\gamma} = \mu_{n+1}$. 
    First, for every $x \in \X$,
    \begin{equation}
    \label{eq:gbound}
	\|g(x)\|^2=\sum_{i=1}^d \ps{k(x,.),g_i}^2_{\cH_0} \le \|k(x,.)\|_{\cH_0}^2 \|g\|_{\cH}^2\le B^2 \|g\|_{\cH}^2,
	\end{equation}
	and
	\begin{align}
	\label{eq:Jgbound}
	\|Jg(x)\|_{\HS}^2&=\sum_{i,j=1}^d \left|\frac{\partial g_i(x)}{\partial x_j} \right|^2\nonumber\\
	&=\sum_{i,j=1}^d \ps{\partial_{x_j}k(x,.), g_i}_{\cH_0}^2 \nonumber \\
	&\le \sum_{i,j=1}^d \| \partial_{x_j}k(x,.)\|^2_{\cH_0} \|g_i\|_{\cH_0}^2 \nonumber\\
	&=\| \nabla k(x,.)\|^2_{\cH}\|g\|^2_{\cH}\nonumber\\
	&\le B^2 \|g\|^2_{\cH}.
	\end{align}
    Hence, 
    \begin{equation}
    \label{eq:invers-glob}
        \|t Jg(x)\|_{\op} \leq \|t Jg(x)\|_{\HS} \leq \gamma B \|g\|_{\cH} \leq \frac{\alpha-1}{\alpha} < 1,
    \end{equation} 
    using our assumption on the step size $\gamma$. Inequality~\eqref{eq:invers-glob} proves that $\phi_t$ is a diffeomorphism for every $t \in [0,\gamma]$. Moreover, 
    \begin{equation}
    \label{eq:alpha}
		\|(J\phi_t(x))^{-1}\|_{\op} \leq \sum_{k=0}^\infty \|t Jg(x)\|_{\op}^k \leq \sum_{k=0}^\infty \left(\frac{\alpha-1}{\alpha}\right)^k = \alpha.
    \end{equation}
    
    Using the density of the pushforward formula,
        \[\rho_t(x) = \left|\det((J\phi_t)^{-1})\mu_n \right|\circ \phi_t^{-1}.\]
        Moreover, $\det(J\phi_t(x))^{-1} \leq \alpha^d$ for every $x \in \X$ using~\eqref{eq:alpha}. Besides, if $\mu_n \in C_0(\X)$ then $\mu_n \circ \phi_t^{-1} \in C_0(\X)$ using that $\phi_t$ is a diffeomorphism. Therefore, $\rho_t \in C_0(\X)$ as the product of a $C_0(\X)$ function with a bounded function. In particular, $\mu_{n+1} \in C_0(\X)$. By induction, $\mu_{k} \in C_0(\X)$ for every $k$.

    Using~\citet[Theorem 5.34]{villani2003topics}, the velocity field ruling the time evolution of $\rho_t$ is $w_t \in L^2(\rho_t)$ defined by $w_t(x) = -g(\phi_t^{-1}(x))$.   
     Denote $\varphi(t) = \cF(\rho_t)$. Using a Taylor expansion,
        \begin{equation}
        \label{eq:Taylor}
            \varphi(\gamma) = \varphi(0) + \gamma \varphi'(0) + \int_{0}^{\gamma} (\gamma - t)\varphi''(t)dt.
        \end{equation}
        We now identify each term. First, 
$\varphi(0) = \cF(\mu_n)$ and $\varphi(\gamma) = \cF(\mu_{n+1})$.
        Using the reproducing property, one can show that
        \begin{equation}
        \label{eq:chainrule1}
            \varphi'(0) =
            -\ps{h_{\mu_n},g}_{\cH}.
        \end{equation}
        


        Moreover, one can show that $\varphi''(t) = \psi_1(t) + \psi_2(t)$, where 
        \begin{equation}
        \label{eq:chainrule2}
        \psi_1(t) = \E_{x \sim \rho_t} \left[ \ps{w_t(x), H_F(x) w_t(x)}\right] \quad \text{and} \quad \psi_2(t) = \E_{x \sim \rho_t} \left[ \|J w_t(x)\|_{\HS}^2 \right].
        \end{equation}
        Recall that $w_t = -g \circ (\phi_t)^{-1}$.
    The first term $\psi_1(t)$ is bounded using the transfer lemma, \Cref{ass:V_Lipschitz} and Inequality~\eqref{eq:gbound}:
    \begin{equation*}
        \psi_1(t) = \E_{x \sim \mu_n} \left[ \ps{g(x), H_V(\phi_t(x)) g(x)}\right] \leq L \|g\|_{L^2(\mu_n)}^2 \leq L B^2 \|g\|_{\cH}^2.
    \end{equation*} 
    For the second term $\psi_2(t)$, using the chain rule, $-J w_t \circ \phi_t = Jg (J \phi_t)^{-1}$. Therefore, 
    \begin{equation*}
        \|J w_t\circ \phi_t(x)\|_{\HS}^2 \leq \|Jg(x)\|_{\HS}^2 \|(J \phi_t)^{-1}(x)\|_{\op}^2 \leq \alpha^2 B^2 \|g\|_{\cH}^2,
    \end{equation*}
    using \eqref{eq:Jgbound} and \eqref{eq:alpha}.
      Combining each of the quantity in the Taylor expansion~\eqref{eq:Taylor} gives the desired result.
\section{Proof of Theorem~\ref{th:svgd}}

We start with a Lemma.

\begin{lemma}
\label{lem:T1}
     Let Assumptions \ref{ass:V_Lipschitz}, \ref{ass:T1} and \ref{ass:k_bounded} hold true. Then, for every $\mu \in \cP_2(\X)$, we have
\begin{equation*}
    \|h_{\mu}\|_{\cH} \leq B \left(1 + \|\nabla F(0)\| + L \int \|x\|d\pi(x)\right) + BL \sqrt{\frac{2\cF(\mu)}{\lambda}}
\end{equation*}
and
\begin{equation*}
\|h_\mu\|_{\cH} \leq B \left(1 + L \sqrt{\frac{2\cF(\mu_0)}{\lambda}} + L \sqrt{\frac{2\cF(\mu)}{\lambda}} + L \int \|x-x_\star\|d\mu_0(x)\right).
\end{equation*}
\end{lemma}
\begin{proof}
Using \Cref{ass:k_bounded}
    \begin{align*}
        \|h_{\mu}\|_{\cH} &= \left\|\E_{x \sim \mu} \left(\nabla F(x)\Phi(x) - \nabla \Phi(x)\right) \right\|_{\cH} \\
        &\leq \E_{x \sim \mu} \left\| \nabla F(x)\Phi(x) - \nabla \Phi(x) \right\|_{\cH}\\
        &\leq \E_{x \sim \mu} \left\| \nabla F(x)\Phi(x)\right\|_{\cH} + \E_{x \sim \mu} \left\|\nabla \Phi(x) \right\|_{\cH}\\
        &= \E_{x \sim \mu} \left\|\nabla F(x)\right\|\left\|\Phi(x)\right\|_{\cH} + \E_{x \sim \mu} \left\|\nabla \Phi(x) \right\|_{\cH} \\
        & \leq B \left( \E_{x \sim \mu} \left\|\nabla F(x)\right\| + 1 \right).
    \end{align*}
    
Using \Cref{ass:V_Lipschitz} and Proposition~\ref{prop:stationary}, $\|\nabla F(x)\| \leq \|\nabla F(0)\| + L\|x\|$. Therefore, using the triangle inequality for the metric $W_1$,
    \begin{align*}
        \|h_{\mu}\|_{\cH} &\leq B \left(1 + \|\nabla F(0)\| + L\int \|x\|d\mu(x) \right)\\
        &= B \left(1 + \|\nabla F(0)\| + L W_1(\mu,\delta_0) \right) \\
        & \leq B \left(1 + \|\nabla F(0)\| + L W_1(\pi,\delta_0) \right) + BL W_1(\mu,\pi).
    \end{align*}
We obtain the first inequality using \Cref{ass:T1}: $W_1(\mu,\pi) \leq  \sqrt{\frac{2\cF(\mu)}{\lambda}}$.

To prove the second inequality, recall that
$
\|h_\mu\|_{\cH} \leq B \left( \E_{x \sim \mu} \left\|\nabla F(x)\right\| + 1 \right).
$
Using \Cref{ass:V_Lipschitz} and \Cref{prop:stationary}, $\|\nabla F(x)\| = \|\nabla F(x) - \nabla F(x_\star)\| \leq L\|x- x_\star\|$. Therefore, using the triangle inequality for the metric $W_1$,
\begin{align*}
    \int \|x-x_\star\|d\mu(x) = W_1(\mu,\delta_{x_\star}) & \leq W_1(\mu,\pi) + W_1(\pi,\mu_0) + W_1(\mu_0,\delta_{x_\star})\\
    &\leq \sqrt{\frac{2\cF(\mu_0)}{\lambda}} + \sqrt{\frac{2\cF(\mu)}{\lambda}} + W_1(\mu_0,\delta_{x_\star}).
\end{align*}
Therefore, 
\begin{align}
\label{eq:cond2}
\|h_\mu\|_{\cH} &\leq B \left(1 + L\int \|x-x_\star\|d\mu(x) \right)\nonumber\\
&\leq B \left(1 + L\sqrt{\frac{2\cF(\mu_0)}{\lambda}} + L \sqrt{\frac{2\cF(\mu)}{\lambda}} + L W_1(\mu_0,\delta_{x_\star})\right).
\end{align}
\end{proof}

Besides, Proposition~\ref{prop:TL} can be applied to SVGD by setting $g = h_{\mu_n} \in \cH$. In this case, we obtain the following descent property if the step size is small enough.
\begin{lemma}
\label{lem:svgd}
Let Assumptions~\ref{ass:V_Lipschitz} and~\ref{ass:k_bounded} hold true. Let $\alpha > 1$ and choose $\gamma > 0$ such that $\gamma \|h_{\mu_n}\|_{\cH} \leq \frac{\alpha-1}{\alpha B}$. Then,
    \begin{equation}
    \label{eq:TL-svgd}
        \cF(\mu_{n+1}) \leq \cF(\mu_{n}) - \gamma\left(1 - \frac{\gamma K}{2}\right)\|h_{\mu_n}\|_{\cH}^2,
    \end{equation}
    where $K = (\alpha^2 + L)B$.
\end{lemma}
Contrary to Inequality~\eqref{eq:TL}, Inequality~\eqref{eq:TL-svgd} is a property of \textit{the SVGD algorithm.} 

Having established Proposition~\ref{prop:TL} and Lemmas~\ref{lem:svgd} and \ref{lem:T1}, we are now ready to formulate and prove our main Theorem~\ref{th:svgd}.

\begin{proof}

    We now prove by induction the first implication of Theorem~\ref{th:svgd}: \eqref{eq:condition-step} $\Rightarrow$ \eqref{eq:TL-svgd-cst}.
    First, if $\gamma > 0$ satisfies \eqref{eq:condition-step}, then, using Lemma~\ref{lem:T1}, $\gamma \|h_{\mu_0}\|_{\cH} \leq \frac{\alpha-1}{\alpha B}$. Therefore, using Lemma~\ref{lem:svgd}, 
    \begin{equation*}
        \cF(\mu_{1}) \leq \cF(\mu_{0}) - \gamma\left(1 - \frac{\gamma K}{2}\right)\|h_{\mu_0}\|_{\cH}^2,
    \end{equation*}
\textit{i.e.}, Inequality~\eqref{eq:TL-svgd-cst} holds with $n = 0$.  Now, assume that the condition~\eqref{eq:condition-step} implies Inequality~\eqref{eq:TL-svgd-cst} for every $n \in \{0,\ldots,N-1\}$ and let us prove it for $n = N$. First, $\cF(\mu_N) \leq \cF(\mu_0)$. Letting $A := B \left(1 + \|\nabla F(0)\| + L \int \|x\|d\pi(x)\right) $, this implies \begin{align*}
    A + BL \sqrt{\frac{2\cF(\mu_N)}{\lambda}} \leq A + BL \sqrt{\frac{2\cF(\mu_0)}{\lambda}}. 
\end{align*}
Therefore, if $\gamma > 0$ satisfies \eqref{eq:condition-step}, then $\gamma \|h_{\mu_N}\|_{\cH} \leq \frac{\alpha-1}{\alpha B}$. To see this, using Lemma~\ref{lem:T1} we obtain
\begin{align*}
\gamma \|h_{\mu_N}\|_{\cH}
&\leq  \gamma \left( A + BL \sqrt{\frac{2\cF(\mu_N)}{\lambda}} \right) \leq \gamma \left( A + BL \sqrt{\frac{2\cF(\mu_0)}{\lambda}} \right) \leq \frac{\alpha-1}{\alpha B}.
  \end{align*}
  Therefore, using Lemma~\ref{lem:svgd}, the condition~\eqref{eq:condition-step} implies Inequality~\eqref{eq:TL-svgd-cst} at step $n=N$:
  \begin{equation*}
        \cF(\mu_{N+1}) \leq \cF(\mu_{N}) - \gamma\left(1 - \frac{\gamma K}{2}\right)\|h_{\mu_N}\|_{\cH}^2.
    \end{equation*}
    Finally, it remains to recall that $\|h_{\mu_N}\|_{\cH}^2 = \KSD^2(\mu_N|\pi)$.
    The proof of the second implication of Theorem~\ref{th:svgd}, \eqref{eq:condition-step-2} $\Rightarrow$ \eqref{eq:TL-svgd-cst}, is similar.
    \end{proof}

    \section{Proof of Corollary~\ref{cor:comp}}
Using Corollary~\ref{cor:conv}, if
  \begin{equation*}
      \gamma \leq \min\left((\alpha-1)\left(\alpha B^2 \left(1 + 2L\sqrt{\frac{2\cF(\mu_0)}{\lambda}} + L \int \|x-x_\star\|d\mu_0(x))\right)\right)^{-1},\frac{2}{B (\alpha^2 + L)}\right),
  \end{equation*}
  then,
  \begin{equation*}
      \KSD^2(\bar{\mu}_n|\pi) \leq \frac{2\cF(\mu_0)}{n \gamma}.
  \end{equation*}

Using~\citet[Lemma 1]{vempala2019rapid}, $\cF(\mu_0) \leq F(x_\star) + \frac{d}{2}\log \left(\frac{L}{2\pi}\right)$. Besides,
\begin{equation*}
    \int \|x-x_\star\|d\mu_0(x) = \E_{X \sim \mu_0} \|X - x_\star\| = \frac{1}{\sqrt{L}} \E_{X \sim \mu_0} \|\sqrt{L}(X - x_\star)\|,
\end{equation*}
and using the transfer lemma and Cauchy-Schwartz inequality,
\begin{equation*}
    \int \|x-x_\star\|d\mu_0(x) = \frac{1}{\sqrt{L}} \E_{Y \sim \cN(0,I)} \|Y\| \leq \frac{1}{\sqrt{L}} \left(\E_{Y \sim \cN(0,I)} \|Y\|^2\right)^{1/2} = \sqrt{\frac{d}{L}}.
\end{equation*}
Therefore,
\begin{align*}
    &(\alpha-1)\left(\alpha B^2 \left(1 + 2L\sqrt{\frac{2\cF(\mu_0)}{\lambda}} + L \int \|x-x_\star\|d\mu_0(x))\right)\right)^{-1} \\
    \geq& (\alpha-1)\left(\alpha B^2 \left(1 + 2L\sqrt{\frac{2}{\lambda}}\sqrt{F(x_\star) + \frac{d}{2}\log \left(\frac{L}{2\pi}\right)} + \sqrt{L d}\right)\right)^{-1} \\
    =& \tilde{\Omega}\left(\frac{1}{\frac{L\sqrt{d}}{\sqrt{\lambda}} + \sqrt{Ld}}\right),
\end{align*}
and 
\begin{equation*}
    \gamma^{-1} = \tilde{\cO}\left(\frac{L\sqrt{d}}{\sqrt{\lambda}} + \sqrt{Ld} + L\right) = \tilde{\cO}\left(\frac{L\sqrt{d}}{\sqrt{\lambda}}\right).
\end{equation*}
Since $\cF(\mu_0) = \tilde{\cO}(d)$,
\begin{equation*}
    \frac{\cF(\mu_0)}{\gamma} = \tilde{\cO}\left(\frac{L{d}^{3/2}}{\sqrt{\lambda}}\right).
\end{equation*}
Let $\varepsilon > 0$. To output the mixture $\bar{\mu}_n$ such that $\KSD^2(\bar{\mu}_n|\pi) < \varepsilon$, it suffices to ensure that $\frac{2\cF(\mu_0)}{n\gamma} < \varepsilon$. Therefore, $n = \frac{2\cF(\mu_0)}{\gamma \varepsilon} = \tilde{\Omega}\left(\frac{L{d}^{3/2}}{\varepsilon\sqrt{\lambda}}\right)$ iterations suffice.

\end{document}